\documentclass{article}

\usepackage{soul}
\usepackage{url}
\usepackage[hidelinks]{hyperref}
\usepackage[utf8]{inputenc}
\usepackage[small]{caption}
\usepackage{graphicx}
\usepackage{amsmath}
\usepackage{amsthm}
\usepackage{booktabs}
\usepackage{algorithm}
\usepackage{algorithmic}
\usepackage{dsfont}
\usepackage{amssymb}
\usepackage[nolist]{acronym}
\usepackage{todonotes}
\usepackage[process=all]{pstool}
\usepackage{mdframed}
\usepackage[russian,english]{babel}
\usepackage{tcolorbox}
\usepackage{breqn}
\renewcommand{\vec}{\mathbf}
\newcommand{\set}[1]{\left\{#1\right\}}

\renewcommand{\span}{\text{span}}
\newcommand{\abs}[1]{\left|#1\right|}
\newcommand{\norm}[1]{\left\|#1\right\|}

\newcommand{\R}{\mathds{R}}
\newcommand{\qq}[1]{``#1''}
\newcommand{\rank}{\mathop{\text{rank}}}
\newcommand{\proj}{\text{proj}}

\newcommand{\N}{\mathds{N}}

\newcommand{\argmin}{\mathop{\text{argmin}}}
\newcommand{\eps}{\varepsilon}

\newtheorem{example}{Example}
\newtheorem{theorem}{Theorem}
\newtheorem{definition}{Definition}
\newtheorem{lemma}{Lemma}
\newtheorem{remark}{Remark}
\newtheorem{corollary}{Corollary}

\title{Supervised Machine Learning with Plausible Deniability}

\newcommand{\email}[1]{\href{mailto:#1}{\texttt{#1}}}
\author{
    Stefan Rass\thanks{Universitaet Klagenfurt, Institut of Artificial Intelligence and Cybersecurity, Universit\"atsstrasse 65-67, 9020 Klagenfurt, Austria, \email{stefan.rass@aau.at}}
    \thanks{Johannes Kepler University, Secure and Correct Systems Lab, Altenberger Straße 69, 4040 Linz, Austria, \email{stefan.rass@jku.at}}
    \and
Sandra K\"onig\thanks{AIT Austrian Institute of Technology, Center for Digital Safety and Security, Giefinggasse 4, 1210 Vienna, Austria, \email{sandra.koenig@ait.ac.at}}
    \and
    Jasmin Wachter\thanks{Universitaet Klagenfurt, Doctoral School for Responsible Safe and Secure Robotic Systems Engineering, Universit\"atsstrasse 65-67, 9020 Klagenfurt, Austria, \email{jawachte@edu.aau.at}}
    \and
    Manuel Egger\thanks{Universitaet Klagenfurt, Institut of Artificial Intelligence and Cybersecurity, Universit\"atsstrasse 65-67, 9020 Klagenfurt, Austria, \email{m8egger@edu.aau.at}}
    \and
    Manuel Hobisch\thanks{Universitaet Klagenfurt, Institut of Artificial Intelligence and Cybersecurity, Universit\"atsstrasse 65-67, 9020 Klagenfurt, Austria, \email{mahobisch@edu.aau.at}}
}

\begin{document}

\maketitle
\begin{abstract}
	We study the question of how well \ac{ML} models trained on a certain data
	set provide privacy for the training data, or equivalently, whether it is
	possible to reverse-engineer the training data from a given \ac{ML} model.
	While this is easy to answer negatively in the most general case, it is
	interesting to note that the protection extends over non-recoverability
	towards \emph{plausible deniability}: Given an \ac{ML} model $f$, we show
	that one can take a set of purely random training data, and from this
	define a suitable \qq{learning rule} that will produce a \ac{ML} model that
	is exactly $f$. Thus, any speculation about which data has been used to
	train $f$ is deniable upon the claim that any other data could have led to
	the same results. We corroborate our theoretical finding with practical
	examples, and open source implementations of how to find the learning
	rules for a chosen set of raining data.
	
\end{abstract}

\section{Introduction}
Imagine a situation in which training data has been used to fit a \ac{ML}
model, which Alice gives away to Bob for his own use. Alice's training data,
however, shall remain her own private property, and Bob should be unable to
recover this information from the \ac{ML} model in his possession. For
example, Alice could be a provider of a critical infrastructure, having
trained a digital twin to emulate the behavior of her system, which Bob, as a
risk analyst, shall assess on Alice's behalf. To this end, however, Alice
must not disclose all the details of her infrastructure, since this is highly
sensitive information and Bob, as an external party, may not be sufficiently
trustworthy to open up to him. Still, Alice needs Bob's expertise on risk
management and risk assessment to help her protect her assets, and therefore
needs to involve Bob to some extent.


We cannot prevent Bob from \qq{guessing}, i.e., Bob can always try to
reverse-engineer the data that Alice used to create the model. This comes to
a perhaps high-dimensional, yet conceptually simple, optimization problem,
which may indeed be tractable with today's computing power. Our goal here is
the proof of two statements about this possibility: First, if the training
data set is ``sufficiently large'' (where the term ``sufficient'' will be
quantified more precisely), Bob cannot unambiguously recover the training
data. Second, and more importantly, Alice can deny any proposal training data
that Bob thinks to have recovered, by exposing a set of random data along
with a certificate that this random decoy data has been used to train the
model (although it was not). Alice can do so by adapting her optimization
problem accordingly to give a desired result (the \ac{ML} model that Bob has)
from any a priori (randomly chosen) training data set.

Note that Bob, since he can \qq{use} the \ac{ML} model, has no difficulties
to evaluate it on a given dataset to produce data upon which a re-training of
the model would reproduce what Bob received from Alice. This trivial
possibility cannot be eliminated. Our question, however, is whether Bob
cannot just produce ``any'' dataset, but find Alice's original dataset that
way used to produce the model in his possession. In other words, does an
\ac{ML} model leak out private information of Alice? The answer obtained in
this work is ``no'', by leveraging a degree of freedom in how an \ac{AI}
model is trained: Alice can provide Bob with decoy data that she claims to
underly what Bob has as the \ac{ML} model; however, Alice can plausibly claim
the model to have come up as the optimum under some optimization problem that
she can craft to her wishes.

The key observation reported in this paper is the fact that we can
``utilize'' non-explainability for the purpose of privacy of data embodied in
an \ac{ML} model. More specifically, we will show how to define an error
metric that makes the learning algorithm converge to any
 target output that we like. We state this intuition more rigorously in
Section \ref{sec:main-idea}, after some necessary preliminary considerations.
In a way, such a designed error metric acts similar to a ``secret key'' in
encryption, only that it accomplishes plausible deniability in our context. A
numerical proof-of-concept is given in Section \ref{sec:evaluation}. Section
\ref{sec:related-work} embeds ours in the landscape of related work and links
the results with issues of the \ac{GDPR}. Section \ref{sec:conclusion} is
devoted to further uses, limitations, ethical considerations and possible
extensions (further expanded in the Appendix).

\subsection{Problem Setting}
Throughout this work, scalars will appear in
regular font, while bold printing will indicate vectors (lower case letters)
or matrices (uppercase letters); for example, the symbols $\vec
A\in\R^{n\times m}$ means an $(n\times m)$-matrix over $\R$. Uppercase
letters in normal font will denote sets, vector spaces, and random variables.
Probability distributions appear as calligraphic letters, like $\mathcal{F}$.
The symbol $X\sim \mathcal{F}$ indicates the random variable $X$ to have the
distribution $\mathcal{F}$.

Let the \ac{ML} model training be the problem to find a
best function $f$ to approximate a given set of $n$ points, called \emph{training data} $(\vec x_i,
y_i)\in\R^m\times \R$ by \qq{minimizing} the error vector $\vec e=(y_1-f(\vec
x_1), y_2-f(\vec x_2),\ldots,y_n-f(\vec x_n))\in\R^n$. The resulting goodness
of fit is later assessed by evaluating $f$ on a (distinct) set of
\emph{validation data}, often providing some error measure to quantify the approximation quality\footnote{We will hereafter have no
	need for the distinction of training and validation data, since our concern
	is exclusively on the training here.}.

The best function $f$ is usually found by fixing its algebraic form, and
tuning some parameters therein by sophisticated optimization methods. Let us
postpone the formal optimization problem until Section
\ref{sec:plausible-deniability-model}, to first state the problem: assume
that we \emph{are given} a trained (fitted) model $f$, but not the training
data. Is there a way to reverse-engineer the training data from $f$ alone?
For example, if we are given access and insight to a trained \ac{NN}, can we
use the weights that we see therein to learn something about the data that
the \ac{NN} has been trained with?

An obvious answer is ``yes'', if we have the training samples at least
partly, since it is straightforward to evaluate $f$ on given values $\vec
x_i$ to recover at least an approximate version of the target value $y_i$, if
it is the only unknown quantity. To avoid such triviality, let us assume that
the training data is \emph{not available} but that we have \emph{white-box
	access} to the machine learning model $f$. This means that we can look into
how $f$ is constructed (i.e., see the weights if it is a \ac{NN}, regression
model, etc.), but have no clue about the data or any parts of it, on which
the model has been trained. This is what we are after, and wish to reverse-engineer.
The case of partial knowledge of the attacker is revisited and discussed in
Section \ref{sec:conclusion}.


\subsection{Some (selected) Applications}
\noindent\textbf{Making Community Knowledge Securely Available}:
Suppose that we want to release data not directly, but \qq{functionally useable} by fitting an \ac{ML} model so that everyone can produce artificial data from $f$, but we do not hereby disclose the original data that $f$ was trained from. This is to retain intellectual property, while still making the knowledge publicly available.

\noindent\textbf{Co-Simulation}: simulations are in many cases domain-specific, e.g., water networks are described using different (physical) mechanisms as traffic or energy networks. Combining these in a co-simulation framework, such as brought up in \cite{schauer_cross-domain_2020}, raises compatibility issues between different simulation models. Fitting \ac{ML} models, say, \acp{NN}, to emulate the outputs of different simulations provides a simple compatibility layer for co-simulation. Plausible deniability is here good for privacy, say, if the physical structure of the simulated process is sensitive information (e.g., a critical infrastructure, uses data related to persons, etc.)

\section{Definitions}
Our formalization of security distinguishes \emph{deniability} from
\emph{plausible deniability}, where the latter notion is stronger.
Informally, deniability of a hypothesis about training data can be understood
as the possibility that there may be another set of training records that
have produced  the same result. To formalize these notions, we first
introduce a generic representation of the machine learning problem. The
following section is not meant as an introduction to the general field, but
to settle the context and symbols in terms of which we state the main results
of this work.

\subsection{Fitting \ac{ML} Models}
We will consider only supervised training in this work. Specifically, we will
view an algorithm to train an \ac{ML} model as a \emph{function} that returns
a parameterized function $f(\cdot; \vec p)$ upon input of the training data
set $\set{(\vec x_1,y_1),\ldots,(\vec
	x_n,y_n)}$, together with a set $\Omega$ of parameters to configure the
training (optimizer). We assume this configuration to be arbitrary, but admit
an unambigious string representation, i.e., $\Omega\subseteq\{0,1\}^*$. The
variable inputs to $f$ herein take the same structure as the training data.
Viewing the training algorithm as a mapping, it is natural to ask for
invertibility of it, and \emph{deniability} then turns out as
non-invertibility. This brings us to the first definition:
\begin{definition}[Machine Learning Model and Training
	Algorithms]\label{def:ml-model} A \emph{machine learning model} is a set $ML$
	of functions $f:\R^m\times\R^d\to\R$, mapping an input $\vec x\in\R^m$ and
	parameter vector $\vec p\in\R^d$ into $\R$.
	
	A \emph{training algorithm} for a machine learning model $ML$ is a function
	$\text{\emph{fit}}\!\!:\R^{n\times (m+1)}\times\Omega\to ML$. This function takes a training data
	matrix $\vec T$ composed from $n$ instances of input/output pairs $(\vec
	x_i,y_i)\in\R^{m+1}$ for $i=1,2,\ldots,n$, and auxiliary information
	$\omega\in\Omega$, to output a (concrete) element $f\in ML$.
\end{definition}

The temporary assumption of $f$ outputting only scalars is here adopted only
for simplicity, and later dropped towards \ac{ML} models with many outputs in
Section \ref{sec:multivariate-ml-models} as Corollary \ref{cor:main}.


The set $ML$ can contain functions of various shape, and is not constrained
to have all functions of the same algebraic structure, although in most
practical cases, the functions will have a homogeneous form. For example,
$ML$ could be (among many more possibilities)
\begin{itemize}
	\item the set of all linear regression models $f(\vec x,\vec p)=\vec
	p^\top\vec x$, where the vector $\vec p$ is the coefficients in the linear
	model. We will use this example in Section \ref{sec:evaluation}.
\item the set of (deep) neural networks with a fixed topology and number of
    layers. The entirety of synaptic weights and node biases then defines
    the vector $\vec p$.
\item the set of support vector machines, in which $\vec p$ is the normal
    vector and bias for the classifying (separating) hyperplane,
\item and many more.
\end{itemize}
In Definition \ref{def:ml-model}, an implicit consistency between the set of
machine learning models $ML$ and the training algorithm is implied by the
(obvious) requirement that (i) the training data needs to have the proper
form and dimension to be useful with the functions in $f$, and (ii) that the
particular element $f$ is specified by an \emph{admissible} parameterization
$\vec p\in\R^d$ for the functions in $ML$, since not all settings for $\vec
p$ may be meaningful to substitute in the general function $f$.

The inclusion of the auxiliary information $\omega$ in the training models
the fact that different models may require different techniques of training,
essentially meaning the application of different optimization techniques. In
particular, $\omega$ will in practical cases (among others) include a
\emph{specification of the error metric} to be used with the training, which
is the goal function to optimize. The core of a training algorithm is a
``learning rule'', being a prescription of how to update the \ac{ML} model
parameterization (iteratively). We will hereafter simplify matters by
abstracting from the detailed optimization technique, and confining ourselves
to look only at the error metric to be used with the optimization, and going
into the training as part of the training algorithm configuration $\omega$.

\subsection{Supervised Training by Optimization}
Generally, we will let the error metric measure the approximation error in a
supervised learning strategy. This learning is based on a set of $n$ samples
$(\vec x_1,y_1),\ldots,(\vec x_n,y_n)\in\R^m\times\R$. In general, the
machine learning problem then takes the generic form of a minimization
problem
\begin{equation}\label{eqn:machine-learning-problem-generic}
	\min\norm{((\vec x_i,y_i)-f(\vec x_i,\vec p))_{i=1}^n}\text{ over }\vec p\in P,
\end{equation}
where the set $P\subseteq \R^d$ optionally constrains parameters to feasible
ranges and combinations. We let $\vec p^*$ denote an (arbitrary) optimum to
this problem, which then pins down a specific $f^*\in ML$. In
\eqref{eqn:machine-learning-problem-generic}, $\norm{\cdot}$ is a topological
norm, specified via the \emph{auxiliary information} $\omega$. Since all
norms on $\R^n$ are equivalent (Theorem \ref{thm:equivalence-of-norms}),
choosing a different norm/error metric only amounts to a scaling of the
(absolute) error bound. Popular error metrics like root mean squared error
(RMSE), mean absolute error (MAE), etc., are all expressible by norms (see
Appendix \ref{sec:error-measures-from-norms} for details omitted here), so
that their use here in place of RMSE, MAE, or others, goes without loss of
much generality. Appendix \ref{sec:error-measures-from-norms} defines norms,
induced metrics and pseudometrics rigorously, for convenience of the reader.

\subsection{Deniability and Plausible Deniability}
Returning to our view of \ac{ML} training as a mere function that, under a
given configuration $\omega$ maps training data to a concrete function $f\in
ML$, we can consider invertibility of this process as the problem of
reverse-engineering the training data from a given model $f$. If this is not possible, in
the sense of (normal) function inversion, then the recovery of training data
from $f$ will fail. Since invertibility is equivalent of simultaneous
injectivity and surjectivity of the training function, the recovery can fail
in two cases:
\begin{enumerate}
	\item the given $f\in ML$ simply does not correspond to \emph{any}
	possible training data under any (or a given) configuration $\omega$.
	In that case, the training algorithm ``fit'' was not surjective, as a
	function.
	\item the given $f\in ML$ may arise identically from several different
	sets of training data, in which case the fitting, as a
	function, was apparently not injective.
\end{enumerate}
It is the latter incident that we will use to define \emph{deniability},
understood as the \emph{possibility} of alternative training data sets,
besides what we have recovered. Formally:


\begin{definition}[Deniability]\label{def:deniability}
	Let a (fixed) $f_0\in ML$ be given
	that has been trained from some unknown data set under a configuration
	$\omega$. We call a given (proposed) training data set $T=\set{(\vec
		x_i,y_i)}_{i=1}^n$ \emph{deniable}, if another set $T'\neq T$ exists, upon
	which the training algorithm \emph{fit} would have produced the \emph{same} function
	$f_0$, possibly under a different configuration $\omega'$ that can depend on
	$T'$.
\end{definition}
Intuitively: \emph{plausibility holds if there is another quantity of
	training data that would have lead to the same $f_0$.}

Obviously, the non-invertibility of the training as a mapping implies
deniability, but the converse is not true, since if the training
function/algorithm is not surjective, no alternative training data $T'$ would
exist. To keep the data recovery problem interesting, however, let us in the
following assume that the model has really been trained from existing yet
unknown information, so that the parameterization is guaranteed to be
admissible.

Even if there is an alternate set of training data, one may question its
validity on perhaps semantic grounds. For example, if the training data is
known to obey certain numeric bounds, or coming from physical processes with
a known distribution, we could perhaps judge an alternative proposal as
implausible, since it \emph{may} produce the same \ac{ML} model, but the
underlying data is arguably not meaningful in the application context. The
stronger notion of \emph{plausible deniability} demands that the alternative
training data should also ``statistically agree'' with the expectations, or
more formally:

\begin{definition}[Plausible Deniability]\label{def:plausible-deniability}
	Let a (fixed) $f_0\in ML$ be given
	that has, under a configuration $\omega$, been trained from some unknown data
	set. Let, in addition, be a distribution family $\mathcal{F}$ be given to describe the context/source of the training data. We
	call a given (proposed) training data set $T=\set{(\vec x_i,y_i)}_{i=1}^n$
	\emph{plausibly deniable}, if another set $T'\neq T$ exists \emph{that has
		the same statistical distribution $\mathcal{F}$}, and upon which the training
	algorithm would have produced the \emph{same} function $f_0$, possibly under
	a different configuration $\omega'$ that can depend on $T'$.
\end{definition}
Intuitively: \emph{plausible deniability holds if it cannot be demonstrated
	that the alternative proposal data is purely artificial.}

Definition \ref{def:plausible-deniability} differs from Definition
\ref{def:deniability} only in the fact that a proposal training data should
not look ``too much different'' from what we would expect about the unknown
training data, formalized by imposing a given distribution $\mathcal{F}$. The
important point here is the order of quantifiers, demanding that the
distribution \emph{family} $\mathcal{F}$ is given a priori, as a
specification of what sort of training data \emph{can be plausible} in the
given context. It is important to observe here that this \emph{does not}
require the unknown data, upon which the given \ac{ML} model $f_0$ has been
trained, needs to have a distribution from $\mathcal{F}$; this can hold in
practical instances, but the denial may indeed be a claim that $f_0$ has been
trained from data coming from an entirely different source, not having the
distribution $\mathcal{F}$. Let us briefly expand on the intuition by giving
an example:

\begin{example}\label{exa:social-network}
	Suppose that in a social network, somebody uses the data from a user to
	predict upcoming messages concerning a certain topic, or just trains a model
	to predict a persons overall activity in posting news on the network. If the
	model is, for simplicity, about the inter-arrival times of a posting on the
	media, we can model the event of postings as a Poisson process, having an
	exponential distribution for the time between two activities with a rate
	parameter $\lambda>0$. Letting $\lambda$ vary over $(0,\infty)$ yields the
	family $\mathcal{F}$ in Definition \ref{def:plausible-deniability}.
	
	Now, suppose that the provider aggregates some statistics about the
	community's activity (say, for advertising purposes), and releases the
	concrete distribution of inter-arrival times between postings to the public
	(e.g., underpinning the empirical findings by releasing artificial data
	coming out of a \ac{GAN} for others to confirm the data science
	independently). This would come to the publication of a specific distribution
	$F_{\lambda}\in\mathcal{F}$ from the aforementioned family of distributions.
	
	Now, to have a need for deniability, one may suspect the provider to have
	profiled a particular network user $X$, and suppose that the activity
	prediction model $f_0$ is about user $X$ specifically. This would be yet
	another member $F_{\lambda_X}\in\mathcal{F}$.
	
	The point behind plausible deniability is that the provider, facing accusal
	of having released an activity model $f_0$ for user $X$, can deny this upon
	admitting that the model \emph{was} trained from social network data, but
	\emph{not} specifically the data of user $X$, having had the distribution
	$F_{\lambda_X}$, but rather from the data for the entire community, having
	the (different) distribution $F_{\lambda}$. The fact that the underlying data
	is admitted to have an exponential distribution is for plausibility, while
	the claim that it was not user $X$'s data is the denial.
\end{example}
While Example \ref{exa:social-network} used the same distribution shape as
the underlying unknown data may have had, a denial may be argued even
stronger by claiming that the distribution used to train $f_0$ may have come
from an entirely different source, having a distinct distribution at all.
Definition \ref{def:plausible-deniability} allows this by not constraining
the distribution family to include only distributions of a particular shape
or algebraic structure (e.g., gamma distribution or more general exponential
family), but allowing it to be any shape that is ``believable'' in the given
context. Our experimental results shown later in Section \ref{sec:evaluation}
demonstrate that this possibility also practically works.

Since this is a much stronger notion than the previous, it comes somewhat
unexpected that it is satisfiable under some conditions, in the sense that we
can even \emph{freely} choose the alternative training data, if we (heavily)
exploit the freedom to change the configuration $\omega$ for the
optimization. In particular, we can modify the error metric, as part of
$\omega$, to let us attain the optimum at the given function $f$ (more
specifically its parameterization $\vec p$) for any a priori chosen training
data. This will be Theorem \ref{thm:main}. Before proving this main result,
let us briefly return to the weaker notion of deniability first. Proving the
possibility to deny is in fact an easy matter of information-theoretic
arguments, as we show in Section \ref{sec:pure-privacy}

\section{Deniability by Non-Unique Recovery}\label{sec:pure-privacy}
Suppose that we are given a model with a (fixed) number of $d$ parameters.
The number $d$ can be large, but still much smaller than the training sample
size, so that there is intuitively no unique recovery possible. In fact, we
have a simple result, whose proof appears in Appendix
\ref{sec:proof-thm-deniability}:
\begin{theorem}\label{thm:deniability}
For a given $ML$ model (according to Definition \ref{def:ml-model}) with $d$
parameters. Let the (unknown) training data come from a random source $Z$
with entropy $H(Z)$ bits, and let the function $f$ require (at least) $k$
bits to encode, and assume that $f$ has been trained from $n$ unknown
records.

If the number $n$ exceeds
\begin{equation}\label{eqn:deniability-bound}
	n> \frac{k}{H(Z)},
\end{equation}
then any candidate training data extracted from $f$ is deniable (in the
sense of Definition \ref{def:deniability}).
\end{theorem}
A suitable number $k$ as used in the above result is practically easy to
find, since it suffices to find \emph{any} number $k$ of bits that encodes
$f$, and if this number is not the minimum, the bound
\eqref{eqn:deniability-bound} only becomes coarser\footnote{finding a tight
	bound in \eqref{eqn:deniability-bound} would require to replace $k$ by the
	entropy of the parameter vector $\vec p$ or the Kolmogorov complexity of the random $f_0$
	as emitted by the training algorithm. Either
	quantity appears hardly possible to get in practice.}. In the simplest
case, $k$ can be found by saving the \ac{ML} model to a file, and taking the
file size to approximate $k$ from above. Expressed boldly, we cannot hope to
extract a ``uniquely defined'' Giga-byte of training data from a 100 kbit
sized model $f$.

\section{Plausible Deniability}\label{sec:plausible-deniability-model}\label{sec:main-idea}
To formalize and prove plausible deniability of the training, imagine an
adversary to have a given model $f_0=f(\cdot, \vec p^*)$ in its possession,
looking to recover the unknown training data $(\vec x_i,y_i)_{i=1}^n$ from
it. For feasibility, let us even assume that the model contains ``enough''
information to let the attacker expect a successful such recovery.
Specifically, the training has lead to the vector $\vec p^*$, from which the
recovery of the data is attempted.

Generically, the recovery is the solution of an inverse (optimization)
problem with $\vec p^*$ as fixed input, and using a norm $\norm{\cdot}$ of
the adversarial reverse-engineer's choice:
\begin{equation}\label{eqn:data-recovery-problem}
	\argmin\norm{((\vec x_i,y_i)-f(\vec x_i,\vec p^*))_{i=1}^n}
\end{equation}
over $\vec (\vec x_i,y_i)_{i=1}^n\in\R^{n\times(m+1)}$ (here being
unconstrained for simplicity and to be clear on the dimensions). Once
confronted with the adversary's proposal solution, the original trainer can
deny the result's correctness by plausibly claiming that the training
algorithm in \eqref{eqn:machine-learning-problem-generic} used a norm that is
\emph{different} from the adversary's choice in
\eqref{eqn:machine-learning-problem-generic}. Theorem \ref{thm:main} gives
conditions under which this claim is possible; more precisely, it lets the
trainer construct a norm from a randomly chosen training data set according
to a desired distribution $\mathcal{F}$, which recovers the model $f$ upon
training with this hand-crafted norm.

Like in encryption, the norm herein takes the role of a ``secret key'' to
train the model, and the plausibility is by exposing a different ``secret''
(norm) to claim that the training was done from entirely different data, and
only coincidentally produced the model in the adversary's hands (Figure
\ref{fig:plausible-deniability-experiment} in the Appendix graphically shows
the flow as an analogy to the secrecy of contemporary encryption; the concept
is comparable).

\subsection{The Main Result}\label{sec:main-theorem} The bottom line of our
previous considerations is that we are thus free to define our error metric
in any way we like, without changing the results of the training in a
substantial way, by crafting our own norm as we desire, and define a distance
metric as the norm of the absolute error vector. In a nutshell, our
construction will use the semi-norm $\norm{\vec x}_A := \sqrt{\vec
x^{\top}\cdot \vec A\cdot \vec x}$, induced by any positive semi-definite
matrix $\vec A$. The trick will be choosing $\vec A$ so that the semi-norm
becomes zero at a desired error vector, i.e., point in $\R^n$. Given any
decoy training data $T'$, it is not difficult to find such a matrix $\vec A$
by computing the error vector $\vec e=(f(\vec x_i,\vec
p^*)-y_i)_{i=1}^n\in\R^n$, and picking $\vec A$ such that $\vec A\cdot \vec
e=0$. Lemma \ref{lem:b-matrix} in Appendix \ref{apx:lemma-b} describes how to
do this step-by-step.

This is almost one half of the construction, culminating in Lemma
\ref{lem:local-optimality}, which adds conditions to ensure the local
optimality of the desired error vector $\vec e$. The other half is the
extension of this semi-norm into a norm, which is Theorem \ref{thm:main}.

\begin{lemma}\label{lem:local-optimality}
	Let $f:\R^m\times\R^d\to\R$ be parameterized by a vector $\vec p\in\R^d$ and
	map an input value vector $\vec x$ to a vector $\vec y=f(\vec x,\vec p)$. Let
	$\vec p^*\in\R^d$ be given as fixed, and let us pick arbitrary training data
	$(\vec x_1,y_1),\ldots,(\vec x_n,y_n)$. 
	Finally, define the error vector $\vec e=(y_i-f(\vec x_i,\vec
	p^*))_{i=1}^n\in\R^n$.
	
	Let for all $\vec x_i$ the functions $f(\vec x_i,\cdot)$ be totally differentiable w.r.t. $\vec p$ at $\vec p=\vec p^*$ with
    derivative $\vec d_i = D_{\vec p}(f(\vec x_i,\vec p))(\vec p^*)\in\R^d$. Put all
    $\vec d_i^\top$ for $i=1,2,\ldots,n$ as rows into a matrix $\vec M\in\R^{n\times d}$
    and assume that it
	satisfies the rank condition
	\begin{equation}\label{eqn:rank-condition}
		\rank(\vec M|\vec e)\neq\rank(\vec M).
	\end{equation}
	
	Then, there exists a semi-norm $\norm{\cdot}$ on $\R^n$ such that $\vec p^*$
	locally minimizes $\norm{e(\vec p^*)}$, i.e., there is an open neighborhood
	$U$ of $\vec p^*$ inside which $\norm{e(\vec p^*)}\leq\norm{e(\vec p)}$ for all $\vec p\in
	U$.
\end{lemma}

\begin{remark}
	The perhaps more convenient condition to work with is assuming $f$ to be
	partially differentiable w.r.t. all parameters $p_1,\ldots,p_d$,
	and to assume the derivatives $\partial f/\partial p_i$ to be continuous at
	all training data points $\vec x_i$. In that case, $\vec d_i$ is just the
	\emph{gradient} $\nabla_{\vec p} f(\vec x_i,\vec p)$ and $\vec M$ is nothing
	else than the \emph{Jacobian} of the function $g:\R^d\to\R^n$, sending $\vec
	p$ to the vector of values $(f(\vec x_1,\vec p),\ldots,f(\vec x_n,\vec p))$,
	where all $\vec x_i$ are fixed, and the result depends only on $\vec p$. The
	general condition stated in Lemma \ref{lem:local-optimality} is just total
	differentiability of $g$, or, in a slightly stronger version, $g$ having all
	continuous partial derivatives.
\end{remark}

The proof of Lemma \ref{lem:local-optimality}, as well as the proof for the
stronger Theorem \ref{thm:main} are both given in the Appendix.


\begin{theorem}\label{thm:main}
Under the hypotheses of Lemma \ref{lem:local-optimality}, there exists a
\emph{norm} $\norm{\cdot}$ on $\R^N$ such that $\vec p^*$
	locally minimizes $\norm{e(\vec p)}$ as a function of $\vec p$.
\end{theorem}
Now, let us go back and remember the order of specification: given the model
by its parameters $\vec p^*$, and -- independently of that -- given an
arbitrary probability distribution family $\mathcal{F}$, we can sample decoy
training data from $\mathcal{F}$, and construct the norm from it. Thm.
\ref{thm:main} thus makes Def. \ref{def:plausible-deniability} of plausible
deniability straightforwardly satisfiable.

It is natural to ask whether the norm that Theorem \ref{thm:main} asserts can
be replaced by a ``more common'' choice of error metric, such as MSE or MAE.
This is in fact possible for MAE; see Appendix \ref{apx:proof:cor:MAE} for
the proof of this Corollary:

\begin{corollary}\label{cor:MAE}
Under the hypotheses of Theorem \ref{thm:main}, there is a matrix $\vec C$
such that $\vec p^*$ locally minimizes the mean average error $MAE(\vec
C\cdot \vec e)$ of the error vector $\vec e$.
\end{corollary}



\subsection{Multi-Output \ac{ML} Models}\label{sec:multivariate-ml-models}

Let us now drop the assumption of our \ac{ML} model to output only numbers,
and look at vectors as output. This transforms the error vector into an error
matrix, and we have the following result, stated again in full detail, and
proven in Appendix \ref{apx:proof:cor:main}.
\begin{corollary}\label{cor:main}
	Take $k,m,d\geq 1$ and let $f:\R^m\times\R^d\to\R^k$ be parameterized by a
	vector $\vec p\in\R^d$, and write $f_j$ for $j=1,\ldots,k$ to denote the
	$j$-th coordinate function. For a fixed parameter vector $\vec p^*$ and
	arbitrary training data $(\vec x_1,\vec y_1),\ldots,(\vec x_n,\vec
	y_n)\in\R^m\times\R^k$, define the error matrix $\vec E$ row-wise as $\vec
	E=(\vec y_i^{\top}-f(\vec x_i,\vec p^*)^{\top})_{i=1}^n\in\R^{n\times k}$. In
	this matrix, let $\vec e_j\in\R^n$ be the $j$-th column.
	
	For all $j=1,2,\ldots,k$ and all training points $\vec x_i$, assume that each $f_j(\vec x_i,\vec p)$ is
    totally differentiable w.r.t. $\vec p$ at (the same point) $\vec p=\vec p^*$, with
    derivative $\vec d_{i,j}=D_{\vec p}(f_j(\vec x_i,\vec p))(\vec p^*)\in\R^d$.
    For each $j$, define the matrix $\vec M_j=(\vec d_{i,j}^\top)_{i=1}^n\in\R^{n\times d}$
    and let the rank
condition $\rank(\vec M_j|\vec e_j)\neq\rank(\vec M_j)$	hold.
	
	Then, there exists a matrix-norm $\norm{\cdot}$ on $\R^{n\times k}$ such that
	$\vec p^*$ locally minimizes $\norm{\vec E(\vec p^*)}$, i.e., there is an
	open neighborhood $U$ of $\vec p^*$ s.t. $\norm{\vec E(\vec p^*)}\leq\norm{\vec
		E(\vec p)}$ for all $\vec p\in U$.
\end{corollary}

Equipped with Theorem \ref{thm:main} and its corollaries, we can now finally
state a result about plausible deniability, similar to Theorem
\ref{thm:deniability}. The proof is by a direct application of the respective
results as stated above.

\begin{theorem}\label{thm:plausible-deniability}
	For a given \ac{ML} model $f$, let the (unknown) training data come from a
	random source with known distribution $\mathcal{F}$. Then, for every choice
	of alternative training data $T'$, randomly sampled from the same
	distribution $\mathcal{F}$, we can find an error metric induced by a
	(properly crafted) norm $\norm{\cdot}$ so that the training algorithm, upon
	receiving the training data $T'$ and error metric (through the configuration
	$\omega$), reproduces the given model $f$ exactly. Thus, any data recovered
	from $f$ is plausibly deniable in the sense of Def.
	\ref{def:plausible-deniability}.
\end{theorem}

The case where the distribution $\mathcal{F}$ is unknown is even simpler,
since plausibility can only be argued if there is a ground truth known as the
distribution $\mathcal{F}$. If this ground truth is not available, there is
nothing to argue regarding plausibility.

\section{Numerical Evaluation and Validation}\label{sec:evaluation}
We demonstrate a proof-of-concept for our plausible deniability concept in
machine learning in the context of a fictional scenario of fitting a
regression model, delegating the (lengthier) details to Appendix
\ref{apx:numerical-example-full}. The experiment was conducted as follows: we
picked a random vector $\vec p$ and defined the \ac{ML} model $f(\vec x)=\vec
p^T\cdot \vec x$ from it. Next, this model was evaluated on randomly chosen
vectors $\vec x_1,\ldots,\vec x_n$, computing the responses $y_i=f(\vec
x_i)+\eps_i$ with a random error term on it. This mimics the model $f$ to
have been fitted from the so-constructed training data $T=(\vec
x_i,y_i)_{i=1}^n$.

Then, towards a denial of the (correct!) training data set, we randomly
sampled a fresh set $T'=(\vec x_i',y_i')_{i=1}^n$, in which the values $y_i'$
were also drawn stochastically independent (of their $\vec x_i'$'s). From
this set $T'$, we constructed the norm as Theorem \ref{thm:main} prescribes
(see Figure \ref{fig:norm-evaluation} in Appendix \ref{apx:proof:cor:main}
for the algorithmic details), and re-fitted the regression model. Plausible
deniability is then the expectation of finding approximately the vector $\vec
p$ again, and indeed, an example execution of this program delivered the
following results for a six-dimensional regression model (small enough for a
visual inspection):

\begin{center}
\begin{tabular}{c|c}
  original vector $\vec p$ & $\vec p$ as trained from decoy data $T'$ \\\hline
  -0.57104  & -0.56936 \\
  -1.53456  & -1.53402 \\
  -2.45770  & -2.45657 \\
  -2.12341  & -2.12261 \\
  -1.26093  & -1.25992 \\
  -1.91170  & -1.91082 \\
\end{tabular}
\end{center}

This experiment is repeatable (with comparably good results) using our
implementation\footnote{code will be released if this paper
	receives positive reviews} of the construction behind Theorem
\ref{thm:main} in GNU Octave (version 5.2.0) \cite{octave5-2-0}, with the
\texttt{optim} package (version 1.6.0) \cite{till_optim_2019}, and for the
particular application to a regression model. We stress that the algorithms
used to fit the \ac{ML} model were hereby taken ``off the shelf'' that
\texttt{optim} provides, with no modification to the inner code (or its
default configuration).



\section{Related Work}\label{sec:related-work}
The conflicting interests of available data and data privacy have long been understood. It has been shown that the problem of minimizing information loss under given privacy constraints is NP-hard \cite{vinterbo_privacy_2004}. An overview on threats and solutions of privacy preserving machine learning is provided in \cite{al-rubaie_privacy-preserving_2019} to close the gap between the communities of \ac{ML} and privacy.\\

Legal requirements such as the \ac{GDPR} put limitations on any kind of method that uses personal data, including \ac{ML} applications. The regulation aims at preventing any discrimination, so critical data such as health data now require protection \cite{azencott_machine_2018}. Approaches such as the privacy-aware machine learning model provisioning platform AMNESIA \cite{stach_amnesia_2020} make sure that \ac{ML} models only remember data they are supposed to remember. A new method to preserve privacy for classification methods in distributed systems prevents that data or the learned models are directly revealed \cite{jia_preserving_2018} and can even be extended to hierarchical distributed systems \cite{jia_efficient_2019}. The vulnerabilities \ac{ML} methods induce in software systems can also be analysed based on known attacks \cite{papernot_sok_2018}.
A recent survey on privacy-preserving \ac{ML} is given in \cite{junxu_survey_2020}, showing that the majority of new approaches focus on specific domains. In social networks, systems are develop that decide (semi-)automatically whether to share information with others \cite{bilogrevic_machine-learning_2016}. Frameworks for privacy-preserving methods in healthcare are also in development \cite{fritchman_privacy-preserving_2018}.
Classification protocols that ensure confidentiality of both data and classifier are described in \cite{bost_machine_2015} and implemented by modification of existing protocols.
In 2017, Google presented a protocol that enables deep learning from user data without learning about the individual user \cite{bonawitz_practical_2017}.
An algorithm for privacy-preserving logistic regression was designed to address the trade-off between privacy and learnability and to learn from private databases \cite{chaudhuri_privacy-preserving_2009}.


\section{Conclusions}\label{sec:conclusion}
\subsection{Suspicion by ``non-standard'' error metrics}
Obviously, it may be suspicious if the norm used for the training is not
released a priori as part of the description of the \ac{ML} model, and our
proposed mechanism of deniability works only if the norm used for the
training is kept secret initially. Furthermore, the honest creator of the
model cannot later come out with a strangely crafted norm to claim having
done the training with this, if the more natural choice would have been MAE,
RMSE or others. So, to make the denial ``work'', the process would require
the model creator to initially state that the training will be done with a
norm that has a ``certain algebraic structure'', namely that which Theorem
\ref{thm:main} prescribes. This lets the honest owner of the norm later
change the appearance of the norm for a denial, without creating suspicion by
coming out with something completely different. Since all vector norms, and
hence also all matrix norms are topologically equivalent, such an a priori
vote for a certain class of norms is not precluded by theory, and a
legitimate design choice up to the model trainer.

\subsection{Accounting for Partial Knowledge}
If the attacker has partial knowledge of the training data, say, a few
columns / variables are known, but not all of them, the situation with
plausible deniability is unchanged: the denying party can simply include this
knowledge in the decoy training data (as this can be chosen freely anyway),
and construct the norm from the remaining variables. This even works when the
attacker knows \emph{all} variables in the training records $\vec x_i$, in
which case the resulting responses $y_i$ are uniquely recoverable by a mere
evaluation of the function $f$. This is the trivial case of recovery, against
which no countermeasure can be given. However, if there is at least some
uncertainty about a variable in the training data, and the model is
``sufficiently dependent'' on this unknown inputs, then plausible deniability
becomes applicable again.

Overall, the finding in this work is that \emph{privacy by
non-recoverability} essentially holds without much ado, provided that there
is lot more data used for the training than the model can embody via its
parameters. Additional precautions for plausible deniability are only
required by announcing the error metric prior to any training, or as part of
the description of the model upon its release.

The important point here is \emph{not} that the training on a suitably
crafted norm is algorithmically feasible, but instead that \emph{it is
possible}. While we do not claim the norm from Theorem \ref{thm:main} to lend
itself to an efficient optimization in high-dimensional cases (such as neural
networks), but the existence assertion made by the theorem may already be
enough, since it is arguable that one has taken the decoy data and went
through very lengthy and time-consuming training to have produced the model
in discussion.

The lesson learned here to escape the plausible deniability issue is to go
for maximum transparency of the learning process, which includes in
particular an \emph{a priori} and \emph{publicly documented} specification of the error metric and training algorithm before deniability arguments are made. In this way, one cannot later silently change the error metric towards consistency with faked training
data.

\section*{Acknowledgments}
This work was supported by the research Project ODYSSEUS ("Simulation und
Analyse kritischer Netzwerk Infrastrukturen in Städten") funded by the
Austrian Research Promotion Agency under Grant No. 873539.

\bibliographystyle{plain}

\begin{thebibliography}{}

\end{thebibliography}


\begin{thebibliography}{10}

\bibitem{al-rubaie_privacy-preserving_2019}
Mohammad Al-Rubaie and J.~Morris Chang.
\newblock Privacy-preserving machine learning: Threats and solutions.
\newblock 17(2):49--58, 2019.

\bibitem{azencott_machine_2018}
C.-A. Azencott.
\newblock Machine learning and genomics: precision medicine versus patient
  privacy.
\newblock 376(2128):20170350, 2018.

\bibitem{bilogrevic_machine-learning_2016}
Igor Bilogrevic, Kévin Huguenin, Berker Agir, Murtuza Jadliwala, Maria Gazaki,
  and Jean-Pierre Hubaux.
\newblock A machine-learning based approach to privacy-aware
  information-sharing in mobile social networks.
\newblock 25:125--142, 2016.

\bibitem{bonawitz_practical_2017}
Keith Bonawitz, Vladimir Ivanov, Ben Kreuter, Antonio Marcedone, H.~Brendan
  {McMahan}, Sarvar Patel, Daniel Ramage, Aaron Segal, and Karn Seth.
\newblock Practical secure aggregation for privacy-preserving machine learning.
\newblock In {\em Proceedings of the 2017 {ACM} {SIGSAC} Conference on Computer
  and Communications Security}, pages 1175--1191. {ACM}, 2017.

\bibitem{bost_machine_2015}
Raphael Bost, Raluca~Ada Popa, Stephen Tu, and Shafi Goldwasser.
\newblock Machine learning classification over encrypted data.
\newblock In {\em Proceedings 2015 Network and Distributed System Security
  Symposium}. Internet Society, 2015.

\bibitem{chaudhuri_privacy-preserving_2009}
Kamalika Chaudhuri and Claire Monteleoni.
\newblock Privacy-preserving logistic regression.
\newblock In D.~Koller, D.~Schuurmans, Y.~Bengio, and L.~Bottou, editors, {\em
  Advances in Neural Information Processing Systems}, volume~21, pages
  289--296. Curran Associates, Inc., 2009.

\bibitem{octave5-2-0}
John~W. Eaton, David Bateman, S{\o}ren Hauberg, and Rik Wehbring.
\newblock {\em {GNU Octave} version 5.2.0 manual: a high-level interactive
  language for numerical computations}, 2020.

\bibitem{fritchman_privacy-preserving_2018}
Kyle Fritchman, Keerthanaa Saminathan, Rafael Dowsley, Tyler Hughes, Martine~De
  Cock, Anderson Nascimento, and Ankur Teredesai.
\newblock Privacy-preserving scoring of tree ensembles : a novel framework for
  \{{AI}\} in healthcare.
\newblock pages 2413--2422. {IEEE}, 2018.

\bibitem{jia_efficient_2019}
Qi~Jia, Linke Guo, Yuguang Fang, and Guirong Wang.
\newblock Efficient privacy-preserving machine learning in hierarchical
  distributed system.
\newblock 6(4):599--612, 2019.

\bibitem{jia_preserving_2018}
Qi~Jia, Linke Guo, Zhanpeng Jin, and Yuguang Fang.
\newblock Preserving model privacy for machine learning in distributed systems.
\newblock 29(8):1808--1822, 2018.

\bibitem{junxu_survey_2020}
Liu Junxu and Meng Xiaofeng.
\newblock Survey on privacy-preserving machine learning.
\newblock 57(2):346, 2020.
\newblock Publisher: Journal of Computer Research and Development.

\bibitem{keras_team_keras_2020}
{Keras Team}.
\newblock Keras documentation: {Losses}, 2020.
\newblock https://keras.io/api/losses/.

\bibitem{papernot_sok_2018}
Nicolas Papernot, Patrick {McDaniel}, Arunesh Sinha, and Michael~P. Wellman.
\newblock {SoK}: Security and privacy in machine learning.
\newblock In {\em 2018 {IEEE} European Symposium on Security and Privacy
  ({EuroS}\&P)}, pages 399--414. {IEEE}, 2018.

\bibitem{schauer_cross-domain_2020}
Stefan Schauer, Sandra König, Thomas Schaberreiter, Stefan Rass, Klaus
  Steinnocher, and Gerald Quirchmayr.
\newblock Cross-{Domain} {Risk} {Analysis} to {Strengthen} {City} {Resilience}:
  the {ODYSSEUS} {Approach}.
\newblock In {\em A.{L}. {Hughes}, {F}. {McNeill} and {C}. {Zobel} (eds.):
  {ISCRAM} 2020 {Conference} {Proceedings} - 17th {International} {Conference}
  on {Information} {Systems} for {Crisis} {Response} and {Management}}, pages
  652--662. ISCRAM Association, 2020.

\bibitem{stach_amnesia_2020}
Christoph Stach, Corinna Giebler, Manuela Wagner, Christian Weber, and Bernhard
  Mitschang.
\newblock \{{AMNESIA}\}: A technical solution towards \{{GDPR}\}-compliant
  machine learning.
\newblock volume Proceedings of the 6th International Conference on Information
  Systems Security and Privacy, pages 21--32, 2020.

\bibitem{till_optim_2019}
Olaf Till.
\newblock The 'optim' package, 2019.

\bibitem{vinterbo_privacy_2004}
S.A. Vinterbo.
\newblock Privacy: a machine learning view.
\newblock 16(8):939--948, 2004.

\bibitem{walter_analysis_1995}
Wolfgang Walter.
\newblock {\em Analysis 2}.
\newblock Grundwissen {Mathematik}. Springer, Berlin, 4., durchges. und erg.
  aufl edition, 1995.
\newblock OCLC: 263611766.

\end{thebibliography}

\newpage

\appendix

\section{Error Measures from Topological
	Norms}\label{sec:error-measures-from-norms}

A norm on $\R^n$ is a mapping $\norm{\cdot}:\R^n\to\R$ with the following properties:
\begin{enumerate}
	\item positive definiteness: $\norm{\vec x}\geq 0$ for all $\vec x$, with $\norm{\vec x}=0$ if and only if $\vec x=0$.
	\item homogeneity: $\norm{\lambda\cdot\vec x}=\abs{\lambda}\cdot\norm{\vec x}$ for all $\lambda\in\R$.
	\item triangle inequality: $\norm{\vec x+\vec y}\leq\norm{\vec x}+\norm{\vec y}$ for all $\vec x,\vec y$.
\end{enumerate}
If one allows $\norm{\vec x}=0$ for some $\vec x\neq 0$, then we call $\norm{\cdot}$ a \emph{semi-norm}. Every norm induces a \emph{metric} $d(\vec x,\vec y)=\norm{\vec x-\vec y}$, or a \emph{pseudometric} if we use a semi-norm.

At least the following popular
choices for error measures are directly expressible via norms. For the
description, let us put $\hat y_i:=f(\vec x_i,p)$ be the \ac{ML} model's
estimate on the training data $(\vec x_i,y_i)$ for a total of
$i=1,2,\ldots,n$ training samples. For abbreviation, put $\vec
y=(y_1,\ldots,y_n),\hat{\vec y}=(\hat y_1,\ldots,\hat y_n)\in\R^n$, and
recall that a general $p$-norm for $p\geq 1$ on $\R^n$ is defined by
\[
\norm{\vec y}_p = \left[\sum_{i=1}^{n}\abs{y_i}^p\right]^{\frac 1 p},
\]
with the practically most important special cases of the 1-norm $\norm{\vec
	y}_1=\sum_{i=1}^{n}\abs{y_i}$, Euclidian norm $\norm{\vec
	y}_2=\sqrt{y_1^2+y_2^2+\ldots+y_n^2}$, and maximum-norm $\norm{\vec
	y}_{\infty}=\max_i\abs{y_i}$.

\begin{enumerate}
	\item Mean squared error
	\begin{equation}\label{eqn:mse}
		MSE = \frac 1 n\sum_{i=1}^{n}(y_i-\hat y_i)^2 = \frac 1 n\norm{\vec y-\hat{\vec y}}_2^2
	\end{equation}
	\item Root mean squared error
	\begin{equation}\label{eqn:rmse}
		RMSE = \sqrt{MSE} = \frac 1 {\sqrt{n}}\norm{\vec y-\hat{\vec y}}_2
	\end{equation}
	\item Mean absolute error
	\begin{equation}\label{eqn:mae}
		MAE = \frac 1 n\sum_{i=1}^{n}\abs{y_i -\hat y_i} = \frac 1 n\cdot\norm{\vec y-\hat{\vec y}}_1
	\end{equation}
\end{enumerate}

We will not go into discussions about pros and cons of these choices (or
alternatives thereto), beyond remarking that the squared errors can be easier
to handle for their differentiability properties. The MAE is on the contrary
more robust against outliers, which the (R)MSE penalize more, so that the
fitting is more sensitive to training data that has not been cleaned from
outliers before.

Defining an error metric from a norm as yet another appeal, since
(topologically) all norms over finite-dimensional real vector-spaces are
equivalent. Since we will make implicit use of that in the following, we
state this well known result for vector-norms, whose canonical version for matrix-norms holds likewise:
\begin{theorem}[see, e.g., {\cite[p.17]{walter_analysis_1995}}]\label{thm:equivalence-of-norms}
	Let any two norms $\norm{\cdot}'$ and $\norm{\cdot}''$ on $\R^n$ be given.
	Then there are constants $\alpha,\beta>0$ such that
	\[
	\alpha\cdot\norm{\vec x}'\leq \norm{\vec x}''\leq\beta\cdot\norm{\vec x}'.
	\]
\end{theorem}
By symmetry, this is an equivalence relation on the set of norms on $\R^n$,
and topologically speaking, they all induce the same topology. For
optimization, it means that once the distance $\norm{\vec x_i-\vec y}\to 0$
as $i\to\infty$ for a point sequence $\vec x_i$ towards approximating a
(fixed) target vector $\vec y$, this convergence would occur in the same way
(though not necessarily at the same speed) in \emph{every other norm} on
$\R^n$.

Practically, this means that fitting a \ac{ML} model to a training data set
by optimizing the norm of the error vector as in
\eqref{eqn:machine-learning-problem-generic}, will eventually lead to results
within a spherical neighborhood (ball) whose radius changes only by a
constant factor upon switching from $\norm{\cdot}'$ to $\norm{\cdot}''$.
Moreover, if an approximation with zero error is possible, both norms will
admit finding this optimum point.

\subsection{Pseudometrics for the Training}\label{apx:lemma-b}
Picking up on the outline started in Section \ref{sec:main-theorem}, a
flexible construction for a norm is $\norm{\vec x}_A := \sqrt{\vec
x^{\top}\cdot \vec A\cdot \vec x}$ with any positive definite matrix $\vec
A$. If $\vec A$ is not positive definite, we can still get a semi-norm as
$\vec x\mapsto \norm{\vec A\cdot\vec x}$, with only the property $\norm{\vec
x}=0\iff \vec x=0$ being violated in case that $\vec A$ has a nontrivial
\emph{nullspace} $N(\vec A)=\set{\vec x: \vec A\cdot\vec x=0}$, where by
nontrivial we mean $N(\vec A)\neq\set{\vec 0}$.

We will proceed by constructing a semi-norm that vanishes only for the given
error vector $e(\vec p^*)$ or scalar multiples thereof, under the chosen
parameter $\vec p^*$. Let us call this particular matrix $\vec B$, whose
existence and construction is not difficult to describe:
\begin{lemma}\label{lem:b-matrix}
	Let $\vec e\in\R^n$ be a vector, then there exists a matrix $\vec B$ having
	the nullspace $N(\vec B)=\span\set{\vec e}$. Geometrically, this matrix is a
	projection on a $(n-1)$-dimensional subspace of $\R^n$, corresponding to the
	orthogonal complement of $\span\set{\vec e}$ within $\R^n$.
\end{lemma}
\begin{proof}
Compute a \ac{SVD} $\vec e=\vec U\cdot\Sigma\cdot \vec V$ for the error
	vector $\vec e$, and construct $\vec B$ with the same rows taken from $\vec
	U^{\top}$ that correspond to all-zero rows (i.e., zero diagonal elements) in
	$\Sigma$. The nullspace and geometric properties then directly follow from
	this construction.
\end{proof}

Using the matrix $\vec B$, we can define the semi-norm
\begin{equation}\label{eqn:seminorm}
	b(\vec x) := \norm{\vec B\cdot\vec x}
\end{equation}
in which $\norm{\cdot}$ is an arbitrary (full) norm on $\R^n$. This is a
well-defined semi-norm, with the properties that
\begin{itemize}
	\item $b(e(\vec p^*))=0$,
	\item and $b(\vec x)>0$ whenever $\vec x\notin\span\set{e(\vec p^*)}$.
\end{itemize}
The function $b$ induces a pseudometric on $\R^n$, as lacking only the
identity of indiscernible elements $d(\vec x,\vec y)=0\iff \vec x=\vec y$,
but still satisfying $b(\vec x)\geq 0$ for all $\vec x$, so that $\vec x=\vec
p^*$ is already an optimum. For later reference, let us capture the matrix
$\vec B$ more explicitly:

The vector $\vec e$ in Lemma \ref{lem:b-matrix} will be our error vector
$e(\vec p)$ for the parameterization $\vec p$, and the subspace that $\vec B$
projects on will be called $V$ throughout all other proofs appearing
hereafter.


Note that, in principle, we could directly use this pseudometric to train our
function $f$ towards taking a minimum error for the parameter $\vec p^*$. The
necessary assumption is that upon a change from $\vec p^*$ to another $\vec
p\neq \vec p^*$, we would leave the nullspace of $\vec B$, thus making the
function $b$ take on strictly positive values.

\section{Proofs}

\subsection{Proof of Theorem
\ref{thm:deniability}}\label{sec:proof-thm-deniability} This is a simple
information-theoretic argument: call $X$ the random variable representing the
(entirety) of the training data that went into the \ac{ML} model. Suppose
this is a set of $n$ records containing values that are sampled from a random
vector $Z$ in a stochastically independent manner. Then, $X$ is a matrix of
$n$ rows, and has the entropy $H(X)=n\cdot H(Z)$, where $H(Z)$ is the entropy
of the joint distribution over the attributes in the training data record.
From here on, let all logarithms have base 2.

The trained model is, from the adversary's perspective, a sample of another
random variable $Y$, representing the collection of parameters that define
the model. The recovery problem is the unique reconstruction of $X$, given
$Y$, and, information-theoretically speaking, solvable if and only if
$H(X|Y)=0$. First, note that $H(X|Y)=H(X,Y)-H(Y)$, and that $H(X,Y)\geq
H(X)$, giving $H(X|Y)\geq H(X)-H(Y)$. Similarly, the information extractable
from the trained model cannot be more than the shortest encoding of the model
itself. So, suppose that the model $f$, as a realization of the random
variable $Y$, comes with a string description of length at least
$K(Y)=\min\{\ell\in\N: f\sim Y$ has an $\ell$ bit string representation$\}$
bits. Then, the uncertainty reduction by $-H(Y)$ cannot exceed the bit count
to represent $f$, hence $H(X)-H(Y)\geq H(X)-K(Y)$. The maximum additional
knowledge of $K(Y)$ bits, contributed by $Y$, is increasing in $d$, since the
parameters at some point must be encoded within the string representation of
$f$. Using this and the fact that $H(X)=n\cdot H(Z)$, with $H(Z)$ being
constant (and determined by the uncertainty in the attributes of the data
that were used for training), we find
\begin{dmath}
H(X|Y)=H(X,Y)-H(Y)\geq H(X)-H(Y)\geq n\cdot H(Z)-K(Y)>0,
\end{dmath}
if the number $n$ of training records grows sufficiently large over the
number $d$ of parameters in the model. Once $H(X|Y)>0$, we have no hope for a
unique recovery of the training data from a model. To be precise, it means
that the distribution is non-degenerate, meaning that there is at least
another \emph{possibility} (i.e., element in the support) to appear with
nonzero probability. This completes the proof of Theorem
\ref{thm:deniability}.

Theorem \ref{thm:deniability} \emph{does not} imply any claim about the
possibility or impossibility to single out a most plausible among the
possible solutions. This would be more likely or easy, the smaller the
conditional or residual entropy comes out, so making $n$ large over $d$ is
practically desirable. Quantifying the chances of guessing is another story,
calling for conditional min-entropies here, and left as a direction of future
research.

While this already positively answers the question of \emph{privacy} of the
data embodied in a \ac{ML} model, this does not rule out a \qq{lucky guess}
of the correct training data. This guess becomes more likely, the smaller the
residual uncertainty $H(X|Y)$ is.

Irrespectively of the residual uncertainty, the stronger possibility of
denying a lucky guess \emph{even if it is correct} is what plausible
deniability is about.


\subsection{Proof of Lemma
\ref{lem:local-optimality}}\label{sec:proof:lem:local-optimality}
	Let $e(\vec p^*)$ be a vector spanning the nullspace of a matrix $\vec B$,
	and let $b$ be defined by \eqref{eqn:seminorm}. Since $f$ is differentiable,
	we can locally write the error
	term as
	\[
	e(\vec p) = e(\vec p^*) + (J_p(f))(\vec p^*)\cdot (\vec p-\vec p^*)+o(\norm{\vec p-\vec p^*})
	\]
	for all $\vec p$ in some neighborhood of $\vec p^*$. Abbreviating our
	notation by writing $\vec M:=(J_p(f))(\vec p^*)$, i.e., calling $\vec M$ the
	Jacobian of $f$ evaluated at $\vec p^*$, and rearranging terms, we get
	\begin{equation}\label{eqn:error-term-linearization}
		e(\vec p)-e(\vec p^*)=\vec M\cdot(\vec p-\vec p^*)+o(\norm{\vec p-\vec p^*}).
	\end{equation}
	
	Towards a contradiction, assume $e(\vec p)\in N(\vec B)$. By construction, we
	have $e(\vec p^*)\in N(\vec B)$, so the difference $e(\vec p)-e(\vec p^*)$ of
	the two is also in $N(\vec B)$. Likewise must thus be the right hand of
	\eqref{eqn:error-term-linearization} in $N(\vec B)$, and we can find a
	sequence $(\vec p_i)_{i\in\N}$ inside $N(\vec B)$ that satisfies
	\eqref{eqn:error-term-linearization}. Because $N(\vec B)=\span\set{e(\vec
		p^*)}$, we can write this sequence as $\vec p_i := \vec p^* + h_i\cdot \vec
	v$, using another null-sequence $(h_i)_{i\in\N}$ of values in $\R$ and the
	unit vector $\vec v:=e(\vec p^*)/\norm{e(\vec p^*)}$ (the norm is herein the
	one from \eqref{eqn:error-term-linearization}, and has nothing to do with the
	one asserted by Theorem \ref{thm:main}). Since the sequence $h_i\to 0$ is
	arbitrary (as is the sequence $\vec p_i$), let us just write $h\to 0$ to
	define the sequence of points in $N(\vec B)$.
	
	This lets us rewrite \eqref{eqn:error-term-linearization} as
	\[
	e(\vec p^*+h\cdot \vec v)-e(\vec p^*)=\vec M\cdot h\cdot \vec v+o(h),
	\]
	which we can divide by $h>0$ to get the quotient
	\[
	\frac{e(\vec p^*+h\cdot \vec v)-e(\vec p^*)}{h}=\vec M\cdot \vec v+\frac{o(h)}{h}.
	\]
	Therein, we have $\frac{o(h)}{h}\to 0$ as $h\to 0$ by the definition of the
	small-o, and on the left hand side, we get the directional derivative along
	$\vec v$ by taking $h\to 0$, since $f$ was assumed to be totally
	differentiable.
	
	Before, we noted the left side of \eqref{eqn:error-term-linearization} to be
	in $N(\vec B)$, and since subspaces are topologically closed, the limit,
	i.e., the directional derivative must also be in $N(\vec B)$. Accordingly,
	this puts the right side $\vec M\cdot \vec v\in N(\vec B)$, implying that
	there is some number $\lambda\in\R$ so that $\vec M\cdot \vec v=\lambda\cdot
	e(\vec p^*)$. But this means that $e(\vec p^*)$ must be in the column space
	of $\vec M$, which contradicts our hypothesis \eqref{eqn:rank-condition} on
	the rank and refutes the assumption that $e(\vec p)$ can be in $N(\vec B)$.
	
	We thus have $e(\vec p)\notin N(\vec B)$ in a neighborhood of $\vec p^*$, but
	$e(\vec p^*)\in N(\vec B)$. Now, using the semi-norm $b(\vec x)=\norm{\vec
		B\cdot\vec x}$, we see that $\norm{e(\vec p^*)}=0$, while $\norm{e(\vec
		p)}>0$, so $\vec p^*$ is locally optimal under this semi-norm.

\subsection{Proof of Theorem \ref{thm:main}}\label{sec:proof:thm:main}

	The norm as claimed to exist above
	will be
	\begin{equation}\label{eqn:b-norm}
		\norm{\vec x} := \norm{\vec x}_e + b(\vec x),
	\end{equation}
	with $b$ as we had so far, and another norm $\norm{\cdot}_e$, to be designed
	later (the subscript $\vec e$ to the norm is hereafter a reminder that this
	norm will depend on the error vector $\vec e$). Intuitively, one may think of
	$b$ as a ``penalty term'' to increase the norm upon any deviation from the
	desired error vector (hence making this point a minimum).
	
	At $\vec p^*$, we have
	\[
	\norm{e(\vec p^*)} = \norm{e(\vec p^*)}_e + \underbrace{b(e(\vec p^*))}_{=0} = \norm{e(\vec p^*)}_e,
	\]
	by our choice of the semi-norm $b$. Our goal is showing that
	\begin{equation}\label{eqn:norm-minimality-goal}
		\norm{e(\vec p^*)}\leq \norm{e(\vec p)}.
	\end{equation}
	From the triangle inequality that $\norm{\cdot}_e$ must satisfy, we get for
	any $\vec p\neq \vec p^*$, $\norm{e(\vec p^*)}_e=\norm{e(\vec p^*)-e(\vec
		p)+e(\vec p)}_e\leq\norm{e(\vec p)}_e+\norm{e(\vec p^*)-e(\vec p)}_e$, and by
	rearranging terms, we find $\norm{e(\vec p)}_e\geq\norm{e(\vec
		p^*)}_e-\norm{e(\vec p^*)-e(\vec p)}_e$. Substituting this into
	\eqref{eqn:b-norm}, we get
	\begin{dmath}\label{eqn:b-norm-lower-bound}
		\norm{e(\vec p)} = \norm{e(\vec p)}_e + b(e(\vec p)) \geq \norm{e(\vec p^*)}_e - \norm{e(\vec p^*)-e(\vec p)}_e + b(e(\vec p)).
	\end{dmath}
	To prove \eqref{eqn:norm-minimality-goal}, it suffices to construct a norm
	$\norm{\cdot}_e$ that satisfies
	\begin{equation}\label{eqn:norm-bound}
		\norm{e(\vec p^*)-e(\vec p)}_e\leq b(e(\vec p)),
	\end{equation}
	for all $\vec p$ for which $e(\vec p)$ is \emph{outside} of $N(\vec B)$
	(otherwise, for $e(\vec p)\in N(\vec B)$ distinct from $\vec p^*$ we would
	have $\norm{e(\vec p^*)-e(\vec p)}>0$ but $b(e(\vec p))=0$, invalidating
	\eqref{eqn:norm-bound}). The assurance that $e(\vec p)\notin N(\vec B)$ is
	hereby implied by the hypothesis and arguments of Lemma
	\ref{lem:local-optimality}, which we included in the theorem's hypothesis
	and hence not repeat here.

	So we can continue \eqref{eqn:b-norm-lower-bound} as
	\begin{dmath*}
	\norm{e(\vec p)} \geq \norm{e(\vec p^*)}_e\underbrace{- \norm{e(\vec p^*)-e(\vec p)}_e + b(e(\vec p))}_{\geq 0}\geq \norm{e(\vec p^*)}_e.
	\end{dmath*}
	With that accomplished, and recalling that $b$ was constructed towards
	$b(e(\vec p^*))=0$, we would find $\norm{e(\vec p)}\geq\norm{e(\vec p^*)}_e =
	\norm{e(\vec p^*)}_e + b(e(\vec p^*)) = \norm{e(\vec p^*)}$, which is exactly
	our goal \eqref{eqn:norm-minimality-goal}.
	
	Thus, we are left with the task of finding a norm $\norm{\cdot}_e$ that
	satisfies \eqref{eqn:norm-bound}. To this end, recall that the semi-norm $b$
	becomes a (full) norm on the factor space $\R^n\slash\sim$, modulo the
	equivalence relation $\vec x\sim \vec y\iff (\vec x-\vec y)\in N(\vec B)$. By
	the dimension formula, we have $\dim(\R^n)=\dim(\R^n\slash\sim)+\dim(N(\vec
	B))$, and since $\dim(N(\vec B))=1$, we find $\dim(\R^n\slash\sim)=n-1$.
	Since the factor space is a vector space over the reals, it is isomorphic to
	the $(n-1)$-dimensional orthogonal complement $V:=N(\vec B)^\bot\subset\R^n$
	of $N(\vec B)\simeq\R^1$. On $V$, we can define a norm, e.g.
	$\norm{\cdot}_2$. By Lemma \ref{lem:b-matrix}, $\proj_V=\vec B$ is the
	projection of a vector onto $V$, then (taking the same norm as in
	\eqref{eqn:seminorm}),
	\[
	\norm{\vec x}_V := \frac 1 2 \norm{\proj_V(\vec x)} = \frac 1 2\cdot b(\vec x)
	\]
	is a semi-norm on $\R^n$. This semi-norm trivially satisfies $\norm{\vec
		x}_V\leq \frac 1 2 b(\vec x)$ for all $x  \in \R^n$.
	Figure \ref{fig:projection-norm} provides an illustration. 

	\begin{figure*}
		\centering
		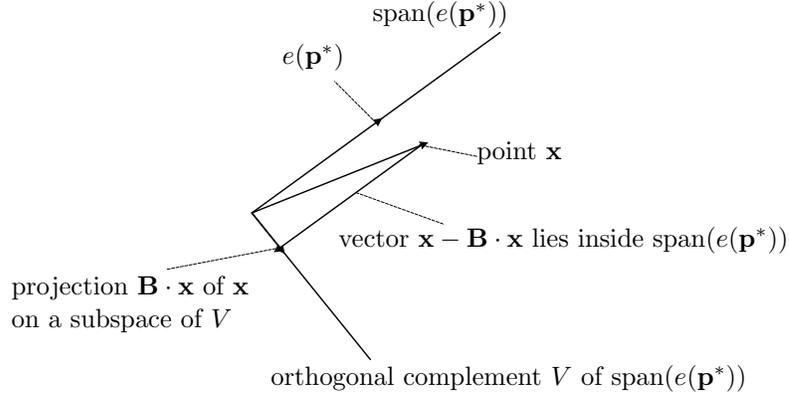
		\caption{Illustration of the projection norm $\norm{\cdot}_V$}\label{fig:projection-norm}
	\end{figure*}
	
	Now, for an intermediate wrap-up, $\norm{\cdot}_V$ is a semi-norm obeying the
	desired bounds for all vectors, especially those in the orthogonal complement
	of $N(\vec B)$, as desired. We now need to extend it to a full norm on the
	entire space $\R^n$ using the following idea: the sum of two semi-norms over
	the same vector space is again a semi-norm and it is a full norm, if and only
	if the intersection of kernels of the two semi-norms is exactly $\{0\}$. So
	we can construct a full norm by adding another semi-norm, that is a full norm
	on a 1-dimensional space (isomorphic to $N(B)$), which retains
	\eqref{eqn:norm-bound} on $\R^n\setminus N(\vec B)$.
	
	The idea is to project a vector in $N(\vec B)$ to the exterior of $N(\vec B)$
	and take the norm of the projection there. To materialize this plan, let
	$\set{\vec v_1,\ldots,\vec v_{n-1}}$ be an orthonormal basis of $V$.
	Furthermore, pick any vector $\vec w_1\in\R^n$ with two properties: (1) it is
	not a scalar multiple of $e(\vec p^*)$, and (2) it is linearly independent of
	all $\set{\vec v_1,\ldots,\vec v_{n-1}}$. In other words, we want both sets
	$\set{\vec w_1,e(\vec p^*)}$ and $\set{\vec w_1,\vec v_1,\ldots,\vec
		v_{n-1}}$ to be linearly independent\footnote{note that $\vec w_1$ is in any
		case \emph{non-orthogonal} to $e(\vec p^*)$, which assures that the
		projection of any element in $\span(e(\vec p^*))$ onto the subspace spanned
		by $\vec w_1$ is nontrivial; if $w_1$ were orthogonal to $e(\vec p^*)$, it
		would necessarily be a scalar multiple of some vector among $\vec
		v_1,\ldots,\vec v_{n-1}$, in which case it cannot be linearly independent of
		them, as we required too.}. An easy choice for $\vec w_1$ is to rotate the
	vector $e(\vec p^*)$ enough to become linearly independent of it, but not far
	enough to become lying in the orthogonal complement. Figure \ref{fig:w1norm}
	graphically sketches the idea formalized now.
	
	\begin{figure*}
		\centering
		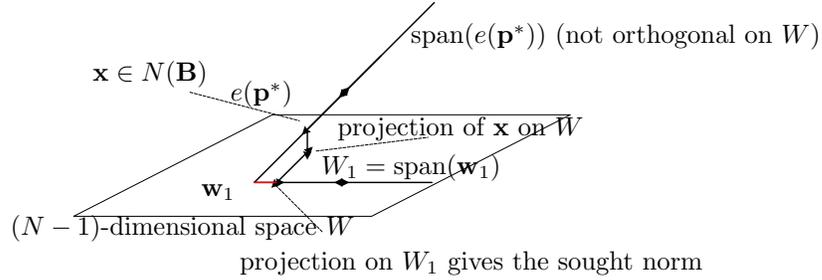
		\caption{Illustration of the construction of $\norm{\cdot}_W$}\label{fig:w1norm}
	\end{figure*}
	
	Call $W_1:=\span\set{\vec w_1}$ the linear hull of $\vec w_1$, and pick
	another $n - 2$ pairwise orthogonal vectors $\vec w_2,\ldots,\vec w_{n-1}$,
	whose entirety spans the space $W_{n-2}^\bot = \span\set{\vec w_2,\ldots,\vec
		w_{n-1}}$ (the subscript and superscript are here serving as reminders about
	the dimensionality and the orthogonality of this space relative to $W_1$).
	Clearly, we have
	\[
	\R^n\simeq W_1\oplus W_{n-2}^\bot \oplus\underbrace{\span\set{e(\vec p^*)}}_{=N(\vec B)}.
	\]
	Now, let any $\vec x\in N(\vec B)$ be given. We can project $\vec x$ on the
	spaces $W_1$ and $W_{n-2}^\bot$. Since the space $W:=W_1\oplus W_{n-2}^\bot$
	is also over $\R$ and has dimension $n-1$, we have the isomorphy
	\[
	W_1\oplus W_{n-2}^\bot\simeq\R^n\slash\sim,
	\]
	so that the function $b$ is again a norm on $W$. Now, let us take the 1-norm
	(an arbitrary choice here) to define another norm on $W$ as
	\[
	\norm{\vec x}_W := \norm{\proj_{W_1}(\vec x)}_1 + \norm{\proj_{W_{n-2}^\bot}(\vec x)}_1.
	\]
	Since all norms over $\R^d$ are equivalent by Theorem
	\ref{thm:equivalence-of-norms} (for all $d$, especially $d=n$ or $d=n-1$),
	there is a constant $\alpha>0$ such that $\alpha\cdot\norm{\vec x}_W<b(\vec
	x)$. By definition of $\norm{\vec x}_W$, we also have $\alpha\cdot\norm{\vec
		x}_1\leq \alpha\cdot\norm{\vec x}_W\leq b(\vec x)$. This lets us define a
	norm on the subspace $W_1\subset W$ as
	\[
	\norm{\vec x}_{W_1} := \frac\alpha 2\cdot\norm{\proj_{W_1}(\vec x)}_1,
	\]
	which satisfies the desired inequality $\norm{\vec x}_{W_1}\leq \frac 1 2
	b(\vec x)$.
	
	Now, let us put together the pieces: define the sought norm $\norm{\cdot}_e$
	as
	\[
	\norm{\vec x}_e:=\norm{\vec x}_V + \norm{\vec x}_{W_1}\leq \frac 1 2 b(\vec
	x) + \frac 1 2 b(\vec x)=b(\vec x),
	\]
	where the inequality is only demanded to hold for $\vec x\notin N(\vec B)$.
	Observe that this is indeed a (full) norm on $\R^n$, since:
	\begin{itemize}
		\item if $\vec x=0$, then $\norm{\vec x}_V=\norm{\vec x}_{W_1}=0$
		\item if $\vec x\neq 0$ and $\vec x\notin N(\vec B)$, then there is a
		nonzero projection $\vec x_V$ on the orthogonal complement of $N(\vec
		B)$, on which $\norm{\vec x_V}_V>0$, and hence $\norm{\vec x}_e>0$.
		Likewise, if $\vec x\neq 0$ and $\vec x\in N(\vec B)$ ($\iff \vec
		x\notin N(\vec B)^\bot$), then there is a nonzero projection on
		$W_1$, making the other part of the norm $>0$.
		\item Homogeneity and the triangle inequality hold by construction and
		are obvious to check.
	\end{itemize}
	Substituting this into \eqref{eqn:b-norm}, we finally get
	\begin{dmath*}
	\norm{e(\vec p)-e(\vec p^*)}_e\leq \norm{e(\vec p)}_e + \norm{e(\vec p^*)}_e \leq b(e(\vec p)) + b(e(\vec p^*)) = b(e(\vec p)),
	\end{dmath*}
	thus satisfying \eqref{eqn:norm-bound}, and yielding the final norm from
	\eqref{eqn:b-norm} as
	\[
	\norm{\vec x} = \frac 3 2 b(\vec x) + \norm{\vec x}_{W_1}.
	\]
    This completes the proof of Theorem \ref{thm:main}. So far, this argument
    is not entirely constructive, but can be made so by reconsidering the
    construction in a little more detail, to which we devote the next
    paragraph.

\subsubsection{Computing the Projections and the Value
	$\alpha$}\label{sec:practicalities} As stated, the proof of Theorem
\ref{thm:main} is not constructive at the point where it claims the
\emph{existence} of the constant $\alpha$ to make $\alpha\cdot\norm{\vec
	x}_{W}\leq b(\vec x)$. Working out a suitable constant $\alpha$ explicitly is
not difficult: every $\vec x\in W_1=\span(\vec w_1)$ takes the form $\vec
x=\lambda\cdot \vec w_1$ for some $\lambda\in\R$, and we can, w.l.o.g.,
assume $\vec w_1$ to have unit length w.r.t. $\norm{\cdot}_1$ on $\R^n$.
Then, $\norm{\proj_{W_1}(\vec x)}_1=\abs{\lambda}$, and $b(\vec
x)=b(\lambda\cdot \vec w_1)=\abs{\lambda}\cdot b(\vec w_1)$. So, it suffices
to choose any $\alpha\in(0,b(\vec w_1))$ to accomplish
$\alpha\cdot\norm{\proj_{W_1}(\vec x)}_1<b(\vec x)$ for $\vec x\in W_1$, as
 desired. If $\vec x\in\R^n$ is arbitrary, its projection is
directly obtained from the standard scalar product $\proj_{W_1}(\vec
x)=\left\langle \vec x,\vec w_1\right\rangle\cdot \vec w_1=(\vec
x^{\top}\cdot \vec w_1)\cdot \vec w_1$ with $\lambda=\left\langle \vec x,\vec
w_1\right\rangle$.

Computing the projection of a vector $\vec x\in\R^n$ on the subspace $V$ is
simply the mapping $\vec x\mapsto\vec B\cdot \vec x$, if $\vec B$ is
constructed as Lemma \ref{lem:b-matrix} prescribes.

Putting together the pieces, given the parameter set $\vec p^*$ and the
resulting residual error vector $\vec e$, the norm as told by Theorem
\ref{thm:main} is explicitly computable along the steps summarized in Figure
\ref{fig:norm-evaluation}.

\begin{figure}
	\begin{mdframed}
		\underline{Input}: Let $\vec e=f(\vec x,\vec p^*)-\vec y\in\R^n$ be the error vector of the ML model $f$
		using the parameters $\vec p^*$, on the training/validation data $(\vec x,\vec y)$.
		
		\underline{Output}: The norm that Theorem \ref{thm:main} speaks about.
		\begin{enumerate}
			\item Compute $\vec B$ as shown in the proof of Lemma \ref{lem:b-matrix}.
			\item Pick a random vector $\vec w_1\in\R^n$ with $\norm{\vec w_1}_1=1$.
			With probability 1, this will deliver a vector that is linearly
			independent of all rows in $\vec B$, and also not a scalar multiple
			of $\vec e$ (but this should nonetheless be checked by checking if the $\vec w_1 \neq \vec{B} \cdot \vec{w}_1$ is fulfilled. Otherwise
			sample another vector $\vec w_1$ and repeat). The probability assurance follows from the
			fact that any lower-dimensional subspace of $\R^n$ has zero Lebesgue
			measure in $\R^n$.
			\item Put $\alpha := \frac 1 2\cdot b(\vec w_1)$, with the function $b$
			defined from the matrix $\vec B$ via \eqref{eqn:seminorm}.
			\item Given any vector $\vec x\in\R^n$, compute the norm $\norm{\vec
				x}_e=\norm{\vec x}_V + \norm{\vec x}_{W_1}$, utilizing that $\norm{\vec
				x}_V=\norm{\proj_V(\vec x)}:=\frac 1 2\cdot b(\vec x)$, and $\norm{\vec
				x}_{W_1}=\frac\alpha 2\cdot\abs{\vec x^{\top}\vec w_1}$, to obtain $\norm{\vec x}$ from \eqref{eqn:b-norm} as
			\begin{equation}\label{eqn:norm-explicit}
				\norm{\vec x} = \frac 3 2 b(\vec x) + \frac\alpha 2\cdot\abs{\vec x^{\top}\cdot\vec w_1}\\
			\end{equation}
		\end{enumerate}
	\end{mdframed}
	\caption{Computation of the norm asserted by Theorem \ref{thm:main}}\label{fig:norm-evaluation}
\end{figure}
	

\subsection{Proof of Corollary \ref{cor:MAE}}\label{apx:proof:cor:MAE}
A re-inspection of the proof of Theorem \ref{thm:main} in Section
\ref{sec:proof:thm:main} quickly shows that it nowhere depends on the
algebraic structure of the function $b$ as given by \eqref{eqn:seminorm}, and
we only used the fact that $b$ is a semi-norm. With that in mind, we can
investigate special cases:


	Define $b$ as
	\begin{equation}\label{eqn:b-alternative}
		b(\vec x) := \norm{\vec B\cdot \vec x}_1,
	\end{equation}
	which has the kernel $N(\vec B)$, and is also a semi-norm. However, it lets
	us express the final norm that Theorem \ref{thm:main} concludes with by a
	more elegant algebraic expression. Upon re-arriving at
	\eqref{eqn:norm-explicit} (see Figure \ref{fig:norm-evaluation}) using the
	function $b$ as defined by \eqref{eqn:b-alternative}, we can expand towards
	\[
	\norm{\vec x}=\frac 3 2\norm{\vec B \vec x}_1+\frac{\alpha}2\abs{\vec x^\top \vec w_1},
	\]
	and, recalling that adding the right term to the 1-norm on the left is the
	same as taking the 1-norm on a vector with merely one additional coordinate,
	we see with a block matrix $\vec C=\binom{(3/2)\cdot\vec B}{(\alpha/2)\cdot
		\vec w_1^\top}$
	\begin{dmath}
	\norm{\vec C\cdot \vec x
	}_1
	=
	\norm{\left(
		\begin{array}{c}
			\frac 3 2 \vec B\cdot \vec x \\
			\frac \alpha 2 \vec w_1^\top\cdot\vec  x \\
		\end{array}
		\right)
	}_1
	= \frac 3 2\norm{\vec B \vec x}_1 + \frac{\alpha}2\abs{\vec x^\top \vec w_1} = \norm{\vec x},
	\end{dmath}
	so that $\norm{\vec e}=\norm{\vec C\cdot \vec e}_1=n\cdot MAE(\vec C\cdot
	\vec e)$ on the error $\vec e$.

This means that the error measured by the norm from Theorem \ref{thm:main} is
\qq{just} the \emph{mean absolute error}, except for a linear transformation
of the error vector. Contemporary machine learning libraries often provide the possibility to define custom loss functions, such as, e.g., \texttt{keras} \cite{keras_team_keras_2020}.


\subsection{Proof of Corollary \ref{cor:main}}\label{apx:proof:cor:main}

	If $f$ is vector-valued with $k$ coordinates, we can apply Theorem \ref{thm:main} to
	each coordinate function $f_j$ for $j=1,\ldots,k$ to obtain a vector norm $\norm{\cdot}_{e_j}$ on $\R^N$ that depends on
	$\vec e_j(\vec p^*)$ and satisfies
	\begin{equation}\label{eqn:norm-optimality}
		\norm{\vec e_j(\vec p^*)}_{e_j}\leq \norm{\vec e_j(\vec p)}_{e_j}
	\end{equation}
	for the parameterization $\vec p^*$ that is the same for all $k$, and all
	$\vec p$ in a neighborhood of $\vec p^*$. From these vector norms, we can define
	\begin{equation}\label{eqn:sum-norm}
		\norm{\vec A}=\sum_{j=1}^{k}\norm{\vec a_j}_{e_j},
	\end{equation}
	with $\vec a_j$ being the $j$-th column in the matrix $\vec A$. This is readily checked to be a matrix-norm, but now works on the
	multivariate error $\vec E=(\vec e_1(\vec p),\ldots,\vec e_k(\vec p))$. The optimality of
	$\vec p^*$ under this norm then directly follows by summing up
	\eqref{eqn:norm-optimality} over $j=1,2,\ldots,k$. This completes the proof.

The practical evaluation of the norm in the multivariate case thus boils down
to an $k$-fold evaluation of norms from Theorem \ref{thm:main} using the
algorithm from Figure \ref{fig:norm-evaluation}, and summing up the results.
Since all matrix norms are likewise to Theorem \ref{thm:equivalence-of-norms}
equivalent, the previous remarks on the freedom to choose any matrix norm for
fitting the \ac{ML} model remains valid.


\section{Example: Regression Model}\label{apx:numerical-example-full} Let us
first illustrate the application of Theorem \ref{thm:main} on a simple linear
regression model. This choice is convenient for both, a closed-form
expressibility of objects like the Jacobian, as well as it can be designed
with only a few number of parameters for a manual check that the resulting
model really comes up almost identical, whether it has been trained with real
or decoy data.

 The overall experiment went as follows, where we let
the data hereafter be \emph{purely artificial} for the mere sake of easy
visual inspection during the computations and in particular regarding the
results:

\begin{enumerate}
	\item The overall regression model is given by a function with parameter
	$\vec p=(\beta_0,\beta_1,\ldots,\beta_{d-1})$
	\begin{equation}\label{eqn:training-model}
		f(\vec x,\vec p)=\beta_0 + \beta_1\cdot x_1+\beta_2\cdot x_2+\ldots+\beta_{d-1}\cdot x_{d-1}+\eps,
	\end{equation}
	in which $\eps$ is a random error term with assumed zero mean. From the
	model, it is evident that $d=m+1$, so that the input vector $\vec
	x\in\R^m$ has one dimension less than $\vec p$. For the experiment, we
	took a uniformly random vector $\vec p\in[-6,+6]^d$ of reals, to define
	an incoming model $f_0$ ``at random''. The magnitude $\pm 6$ is herein an
	arbitrary choice, to keep the numbers feasibly small for a manual visual
	inspection later.
	
	\item Equation \eqref{eqn:training-model} was then evaluated on a total
	of $n=10$ uniformly random samples $\vec
	X_i\sim\mathcal{U}(\set{1,2,\ldots,8}^m)$, adding stochastically
	independent error terms $\eps$, each with an exponential distribution
	with rate parameter $\lambda=5$ (to, say, let the data be
	inter-arrival times, with an eye back on Example
	\ref{exa:social-network}). Again, the choice of $x$-values in the
	integer range $1,\ldots, 8$ is arbitrary, and only to keep the
	numbers small for a visual checkup. This computation delivers the
	values $y_i\gets f(\vec x_i)+\eps$ for $i=1,2,\ldots,10$, which,
	together with the $\vec x_i$ form the \emph{training data}.
	\item Next, we ``forget'' about the underlying model (that we know here)
	and fit a regression model of the same structure, given only the
	training data. Since this data originally came out of a regression
	model, this lets us expect a quite good fit, and an approximate
	re-discovery of the same parameter vector $\hat{\vec p}$ as we had
	for producing the training data. Deviations are equally natural (yet
	at small scale), since the training data is not overly extensive.
	
	The resulting model $f_0$ is obtained by invoking a nonlinear
	optimization via a call to \verb|nonlin_min|, to minimize the functional
	$\norm{(f(\vec x_i,\vec p)-y_i)_{i=1}^{10}}_2$ using vectorization in
	GNU Octave. The minimization using the 2-norm has, in our case, the appeal of
	making the resulting model a best linear unbiased estimator by the
	Gauss-Markov theorem, whose hypotheses are here satisfied by
	construction. Thus, the trained model $f_0$ is indeed a ``good'' \ac{ML}
	model, as could be expected in real-life applications.
	\item Now, for a plausible denial, we took a fresh set of (stochastically
	independent) samples of \emph{decoy} training data $\vec
	X'_i\sim\mathcal{U}(\set{1,\ldots,8}^m)$, \emph{and} another set of
	random, and hence unrelated, response values $\vec
	Y_i'\sim\mathcal{U}(\set{1,\ldots,8}^m)$. Two things are important
	to note here:
	\begin{itemize}
		\item The decoy data is picked stochastically independent and at
		random, so the experiment was repeatable with different
		instances of all ingredients (only retaining fixed numeric
		ranges for the values),
		\item and, more importantly, the response values $y_i$ are
		\emph{independent} of the inputs $x_i$, so any underlying
		functional relation between $\vec x_i$ and the corresponding
		$y_i$ is most likely not a linear regression model. Thus, the
		decoy data is completely different from the true training
		data.
	\end{itemize}
	\item Given the set of decoy samples $(\vec x_i,y_i)_{i=1}^{10}$, we
	proceed by implementing the steps as shown in Figure
	\ref{fig:norm-evaluation}, producing the GNU Octave local variables
	\texttt{B}, \texttt{w1} corresponding to $\vec B$ and $\vec w_1$ from
	the text, and implementing the norm that Theorem \ref{thm:main}
	constructs as a function \verb|crafted_norm|. All these computations
	take less than 10 lines of code\footnote{In Octave only, but a port
		to Python or other languages is not expected to become considerably
		more complex.}.
	
	For checking the hypothesis of Lemma
	\ref{lem:local-optimality}, i.e., the rank condition
	\eqref{eqn:rank-condition}, the regression model comes in handy once more:
	it allows for a closed form expression of the Jacobian at $\vec p$, given
	directly by the data matrix, augmented with a mere column of all 1es,
	i.e., for our model $f(\vec
	x,(\beta_0,\ldots,\beta_{d-1}))=\beta_0+(\beta_1,\ldots,\beta_{d-1})\cdot\vec
	x$, we find the Jacobian to be constant\footnote{More complex models
		would require a manual approximation of the Jacobian (unless analytic
		expressions are obtainable), but this amounts to nested for loop over
		$i=1\ldots n$ and over $j=1\ldots d$ to approximate the derivative
		$\partial f_i/\partial p_j\approx \frac 1 h\cdot (f(\vec x_i,\vec
		p+h\cdot \vec u_j)-f(\vec x_i,\vec p)$, in which $\vec u_j$ is the $j$-th
		unit vector, and $h>0$ is some (very) small constant. This requires the
		\ac{ML} model, as a programming object, has access routines to get and
		set the model parameters as we wish (the regression model is again
		convenient here, since it is easy to implement).}, and given as
	\[
	\vec J = \left(
	\begin{array}{cc}
		1 & \vec x_1 \\
		1 & \vec x_2 \\
		\vdots & \vdots \\
		1 & \vec x_n \\
	\end{array}
	\right),
	\]
	in which each row $\vec x_i$ is the $i$-th data sample used to train the
	model. This is the matrix against we check the rank to change when
	attaching the vector $\vec e$.
	\item With these items, we then go back into the nonlinear optimization,
	again using the same function \verb|nonlin_min|, but this time
	minimizing our designed norm implemented in the function
	\verb|crafted_norm|, and formally found as Figure
	\ref{fig:norm-evaluation} tells.
\end{enumerate}
The results, quite satisfyingly, demonstrated that the model fitted to the
decoy data but using the specially constructed norm comes up approximately
equal to the original model. Notably, it does so with the decoy data having
no relation to the training data whatsoever, not even necessarily sharing its
original distribution (the original data was a linear combination of uniform
distributions, which is no longer uniform for two or more terms, while the
decoy data had an overall uniform distribution). The numeric discrepancies
between the newly fitted model and the original model can partly be
attributed to our lack of fine-tuning in the optimization process; indeed, we
invoked \verb|nonlin_min| with all \emph{default} settings, except for the
starting point to be inside a neighborhood of the given parameter vector
$\vec p$, known from the given model $f_0$. Indeed, even in the default
configuration, the model fitted under the true and the decoy data came up
quite ``close'' to each other, indicating potentially higher accuracy upon
careful fine-tuning of the optimization. In addition, the choice of $\vec
w_1$ may also have an impact on the numeric behavior of the optimizer, as
does any randomness that the optimization algorithm may employ internally. We
leave both possibilities for numeric accuracy gains aside here, leaving the
demonstration with the pointer towards the observation that higher
dimensionality of the model (and we conducted further experiments with larger
values for $d$) made the approximation worse. Again, this is not unexpected
in light of higher-dimensional optimization problems generally behaving less
nice than lower-dimensional ones. Our choice of $d=6$, however, makes a
manual check of equality among 6 pairs of model parameters quick and simple
to show in Section \ref{sec:evaluation}.


\section{A ``Cryptographic'' View}

The flow in Figure \ref{fig:plausible-deniability-experiment} resembles an
analogous situation as for probabilistic encryption, where the norm is
\emph{playing the role of a random auxiliary input} to the encryption
function: let $E_{pk}(m_0,\omega)$ denote the probabilistic encryption of a
message $m_0$ under a public key $pk$ and a random string (random coins)
$\omega$. Given a ciphertext $c$, one could deny the validity of any proposed
plaintext $m_1$ if $\forall c~\exists m,\omega: E_{pk}(m, \omega)=c$. This is
indeed the case for ElGamal encryption (for example). This is the common way
of defining security of encryption (see any of the standard cryptography
textbooks), and our notion of plausible deniability is completely analogue to
this.

\begin{figure}
	\centering
	\includegraphics[scale=0.7]{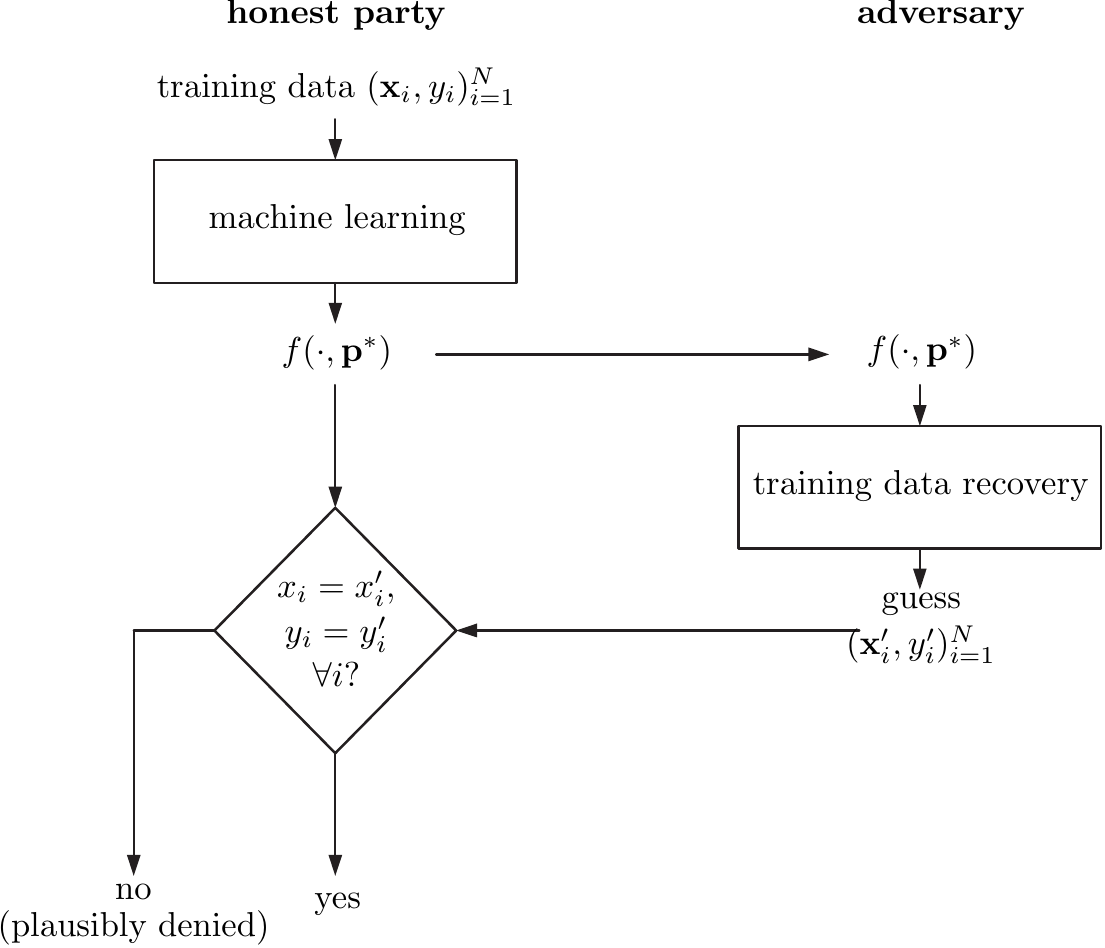}
	\caption{Plausible Deniability Experiment}\label{fig:plausible-deniability-experiment}
\end{figure}

\begin{acronym}
	\acro{ML}{machine learning}%
	\acro{AI}{artificial intelligence}%
	\acro{NN}{neural network}%
	\acro{SVD}{Singular Value Decomposition}%
	\acro{GAN}{Generative Adversarial Networks}%
	\acro{PUF}{Physically Uncloneable Features}%
	\acro{GDPR}{General Data Protection Regulation}%
\end{acronym}

\end{document}